\documentclass{article}
\usepackage[T1]{fontenc}

\usepackage{amsmath,amssymb,amsthm}
\usepackage{booktabs}
\usepackage{graphicx}
\usepackage{xcolor}
\usepackage{hyperref}
\usepackage{microtype}
\usepackage{placeins}
\usepackage[numbers,sort&compress]{natbib}
\usepackage[preprint]{icml2026}

\hypersetup{colorlinks=true,citecolor=blue!55!black,linkcolor=blue!55!black,urlcolor=blue!55!black,pdfsubject={Preprint}}
\graphicspath{{figures/}}

\newtheorem{theorem}{Theorem}
\newtheorem{proposition}[theorem]{Proposition}
\newtheorem{corollary}[theorem]{Corollary}

\theoremstyle{definition}
\newtheorem{definition}[theorem]{Definition}
\newtheorem{assumption}[theorem]{Assumption}
\theoremstyle{remark}

\newcommand{\Hid}{\mathcal{H}}
\newcommand{\State}{\mathcal{S}}
\newcommand{\Tests}{\mathcal{W}}
\newcommand{\Audit}{\mathcal{U}}
\newcommand{\Domain}{\mathcal{D}}
\newcommand{\E}{\mathbb{E}}
\newcommand{\R}{\mathbb{R}}
\newcommand{\norm}[1]{\left\lVert#1\right\rVert}
\newcommand{\abs}[1]{\left\lvert#1\right\rvert}

\newcommand{\DFAHardError}{0.187}
\newcommand{\DFASoftError}{0.254}
\newcommand{\DFAUncorrectedError}{0.670}
\newcommand{\DFAHardContraction}{0.807}
\newcommand{\FiberSemanticRMSE}{0.108}
\newcommand{\GlobalSemanticRMSE}{0.574}
\newcommand{\FiberNormalRMSE}{0.000}

\newcommand{\MixedFiberZRMSE}{0.118}
\newcommand{\MixedGlobalZRMSE}{0.193}
\newcommand{\MixedQuantizedZRMSE}{0.129}
\newcommand{\ScalingExponentOne}{3.14}
\newcommand{\ScalingExponentTwo}{4.18}
\newcommand{\MinimumCertificate}{0.052}
\newcommand{\LowBudgetCertificate}{-0.011}
\newcommand{\HighBudgetCertificate}{0.065}
\newcommand{\CertificationTau}{0.0078}
\newcommand{\InitialCEGISCertificate}{-0.055}
\newcommand{\RefinedCEGISCertificate}{0.066}
\newcommand{\CEGISXi}{0.007}
\newcommand{\CEGISTau}{0.009}
\newcommand{\InitialAuditDistortion}{0.267}
\newcommand{\FinalAuditDistortion}{0.008}
\newcommand{\ReanchoringLawError}{0.000}

\newcommand{\HARModernSeeds}{5}
\newcommand{\HARModernGRUError}{0.126\pm0.007}

\newcommand{\HARModernPFXi}{1.83e-15}
\newcommand{\HARModernGlobalXi}{0.107}
\newcommand{\HARModernContraction}{0.220}

\newcommand{\DigitsPFError}{0.101}
\newcommand{\DigitsPFErrorStd}{0.014}
\newcommand{\DigitsGlobalError}{0.863}
\newcommand{\DigitsGlobalErrorStd}{0.003}

\newcommand{\DigitsSpectralError}{0.121}
\newcommand{\DigitsPFInterface}{1.254e-15}

\newcommand{\DigitsPFContraction}{0.220}

\newcommand{\DigitsEnrichedMargin}{0.973}

\newcommand{\MujocoCertificate}{0.396}
\newcommand{\MujocoBudget}{337{,}500}
\newcommand{\MujocoGRUTau}{0.141}
\newcommand{\MujocoLSTMTau}{0.147}
\newcommand{\MujocoSSMTau}{0.168}
\newcommand{\MujocoPFXi}{8.33e-16}
\newcommand{\MujocoGlobalXi}{0.076}
\newcommand{\MujocoNormalContraction}{0.220}

\newcommand{\DoubleCertificate}{0.396}

\newcommand{\MujocoCEGISInitial}{-0.202}
\newcommand{\MujocoCEGISFinal}{0.398}

\newcommand{\EfficiencyKFourBudget}{10{,}274{,}897{,}024}
\newcommand{\EfficiencyPhysicalLowSeconds}{10.8}
\newcommand{\EfficiencyPhysicalHighSeconds}{44.3}
\newcommand{\EfficiencyPhysicalBudget}{337{,}500}
\newcommand{\EfficiencyPhysicalCertificate}{0.396}
\newcommand{\EfficiencyCEGISMilliseconds}{0.2}

\icmltitlerunning{Predictive Fibers}

\begin{document}

\twocolumn[
\icmltitle{What Can a Recurrent State Safely Forget?}
\begin{icmlauthorlist}
\icmlauthor{Linzhe Zhang}{neu}
\icmlauthor{Changming Xu}{neu}
\end{icmlauthorlist}
\icmlaffiliation{neu}{Graduate School, Northeastern University}
\icmlcorrespondingauthor{Linzhe Zhang}{cfmy007@gmail.com}
\icmlcorrespondingauthor{Changming Xu}{changmingxu@neuq.edu.cn}
\icmlkeywords{recurrent neural networks, predictive state, robustness, certification}
\vskip 0.3in
]
\printAffiliationsAndNotice{}

\begin{abstract}
Recurrent models must preserve information that changes future behavior while suppressing hidden-state error.  These objectives conflict: contraction improves stability, but contraction of a future-distinguishing direction destroys memory.  We formalize the boundary through the \emph{predictive quotient} of a recurrent state space.  Two hidden states are equivalent when they induce the same complete conditional future; their equivalence classes are predictive fibers.  Every exact semantics-preserving corrector acts as the identity on this quotient.  At a regular point with hidden dimension $d$ and predictive dimension $k$, it can eliminate at most $d-k$ independent directions.  This yields a discrete--continuous boundary: finite predictive states can have positive-radius exact correction basins, whereas an uncountable continuum of future-distinguishable states cannot be exactly decoded after arbitrary positive-radius perturbations in finite-dimensional Euclidean space.

To operationalize this principle in learning, we develop an auditable finite-future framework.  A compact deployment bank $W$ is evaluated against an independent audit bank $\Audit\supseteq W$ on a declared correction domain.  Under generative probe access and audit-metric coverage, finite stochastic rollouts furnish a high-probability lower certificate for the separation margin $\Omega^{\Domain}_{W\mid\Audit}(\delta)$.  Preserving learned $W$-predictions within this certified margin guarantees bounded audit-semantic distortion.  For intrinsic audit dimension $k$, the required number of probe outcomes scales as $\widetilde O(M\Omega^{-(k+2)})$, where $M=\abs{\Audit}$ reflects the audit scope; a matching minimax lower bound proves this exponent is optimal.  Extending guarantees across full continuous futures is achieved via an explicit completeness modulus.  Controlled experiments validate the certified margins, their scaling behavior, and automated probe refinement under a safety-first evaluation paradigm.
\end{abstract}

\section{Introduction}

A recurrent state summarizes a history through an update $h_t=F_\theta(h_{t-1},x_t)$.  It is asked to do two incompatible things.  It must retain every aspect of the past that changes the future, while also rejecting noise, numerical drift, and irrelevant variability.  Global contraction solves the second problem by also erasing the first.  The issue is not whether a recurrent state should contract, but \emph{which directions may contract without changing the computation}.

Predictive-state representations define state using conditional predictions of future tests rather than an unobserved latent coordinate \citep{littman2001predictive}.  Neural descendants such as PSRNNs and predictive-state decoders already use future prediction to shape recurrent representations \citep{downey2017psrnn,venkatraman2017psd}.  Continuous-attractor analyses similarly distinguish weakly stable memory directions from strongly stabilized transverse directions \citep{sagodi2024continuous}.  These lines of work motivate, but do not state, a safety criterion for modifying an internal recurrent representation.

This paper makes that criterion explicit.  A hidden-state displacement is safe to remove exactly when it stays inside a \emph{predictive fiber}: the set of hidden states that cannot be distinguished by any relevant future.  The resulting contributions are:

\begin{itemize}
\item a quotient constraint showing that an exact corrector is the identity on predictive state, and a codimension bound on correctable directions;
\item a sharp contrast between positive-radius correction for finite predictive codes and its impossibility for a predictive continuum, plus an approximate packing bound;
\item a finite, non-circular audit formulation and a PAC lower certificate for predictive separation under stated generative access, with structural modulus conditions for full continuous futures;
\item upper and minimax lower bounds with the same $\Omega^{-(k+2)}$ exponent; and
\item audited Predictive-Fiber CEGIS (PF-CEGIS), which enlarges a deployment probe bank only when an independent audit exposes an unsafe alias.
\end{itemize}

\section{The Geometry of Safe Forgetting}\label{sec:geometry}

Let $\Hid\subseteq\R^d$ be a recurrent hidden-state domain.  A future experiment (or test) $w\in\Tests$ has bounded outcome $Y_w\in[0,1]$.  Let $\phi_w(h)=\E[Y_w\mid h]$ and define the complete predictive behavior $\Phi_\infty(h)=(\phi_w(h))_{w\in\Tests}$.  Two hidden states are predictively equivalent ($h\sim h'$) precisely when $\Phi_\infty(h)=\Phi_\infty(h')$.  The quotient $\State=\Hid/{\sim}$ is the predictive state space, and the canonical projection $\pi:\Hid\to\State$ maps each internal realization to its behavioral equivalence class.  An equivalence class $\mathcal F_s=\pi^{-1}(s)$ is a \emph{predictive fiber}.  By construction, all future predictions factor through the quotient state $s=\pi(h)$; we write $\phi_w(s)$ without ambiguity.

\begin{definition}[Semantics-preserving corrector]
A map $Q:\Hid\to\Hid$ is semantics preserving on a domain $D\subseteq\Hid$ if $\pi(Q(h))=\pi(h)$ for all $h\in D$.  Equivalently, $Q$ maps every hidden state into its own predictive fiber.
\end{definition}

\begin{theorem}[Predictive quotient constraint]\label{thm:quotient}
Let $Q$ be semantics preserving.  It induces the identity map on the predictive quotient.  Suppose that near $h$, $\pi$ is represented by a $C^1$ submersion of rank $k$, and that $Q(h)=h$.  Then
\[
 \begin{gathered}
 D\pi_h(DQ_h-I)=0,\\
 \operatorname{rank}(DQ_h)\ge k,\qquad \dim\ker DQ_h\le d-k.
 \end{gathered}
\]
Consequently, at most $d-k$ independent local directions can be eliminated by an exact corrector.  If a $C^1$ local section $E:\State\to\Hid$ exists, the canonicalizer $Q^\star=E\circ\pi$ is an idempotent corrector with rank $k$ along $E(\State)$.
\end{theorem}

\begin{proof}
The first claim is simply $\pi\circ Q=\pi$.  At a fixed point, differentiate this equality to obtain $D\pi_hDQ_h=D\pi_h$; rearrangement gives the first display.  Since $D\pi_h=D\pi_hDQ_h$ and $D\pi_h$ has rank $k$, $DQ_h$ has rank at least $k$.  For a section, $\pi\circ E=I$, so $Q^\star\circ Q^\star=Q^\star$ and $DQ^\star$ has rank $k$ at canonical states by the chain rule.
\end{proof}

Geometrically, Theorem~\ref{thm:quotient} establishes that an exact corrector may aggressively contract redundant representational directions within a fiber, but must act as the identity along the predictive quotient.  The fixed-point condition in the differential identity is structurally exact: away from a fixed point, differentiating $\pi\circ Q=\pi$ yields $D\pi_{Q(h)}DQ_h=D\pi_h$, reflecting the pushforward along non-trivial corrector trajectories.  Crucially, this codimension bound is geometric rather than architectural: it applies universally to learned RNNs, continuous attractors, or physical state estimators.

\subsection{The Discrete--Continuous Boundary}

Whether such semantics-preserving contraction can achieve exact error recovery hinges fundamentally on the topological cardinality of the predictive quotient $\State$.  When $\State$ is finite or discrete, distinct fibers admit disjoint, positive-radius basins of attraction in $\R^d$, enabling robust hard-decision snapping.  However, this discrete intuition completely breaks down for an uncountable predictive continuum:

\begin{theorem}[No positive-radius exact decoder for a predictive continuum]\label{thm:nogo}
Let $\State$ be an uncountable set of pairwise future-distinguishable states and $E:\State\to\R^d$ an encoding.  There do not exist a deterministic decoder $C:\R^d\to\State$ and $\epsilon>0$ satisfying
\[
 C(E(s)+\eta)=s\quad\text{for all }s\in\State\text{ and }\norm{\eta}\le\epsilon.
\]
\end{theorem}

\begin{proof}
The closed $\epsilon$-balls about two distinct codewords must be disjoint: a point in their intersection would require two decoder outputs.  Hence the codewords are pairwise more than $2\epsilon$ apart.  A positive-separated subset of $\R^d$ is countable (partition space into bounded cubes, each of which contains only finitely many such points), contradicting uncountability.
\end{proof}

Theorem~\ref{thm:nogo} delineates a fundamental geometric boundary: while discrete codes admit positive-radius basin snapping, an uncountable continuum in finite-dimensional Euclidean space strictly precludes exact state recovery under arbitrary perturbations.  Consequently, continuous predictive memory must tolerate approximate semantic retention within a certified distortion margin $\delta$.  In this approximate regime, the required hidden dimension is governed by predictive metric entropy:

\begin{theorem}[Predictive packing bound]\label{thm:packing}
Suppose $E(\State)\subseteq[-R,R]^d$ and a decoder $C$ obeys $d_\State(C(E(s)+\eta),s)\le\delta$ for all $\norm{\eta}_\infty\le\epsilon$.  Then
\[
 M_\State(2\delta)\le\left(1+\frac{R}{\epsilon}\right)^d,
 \qquad
 d\ge\frac{\log M_\State(2\delta)}{\log(1+R/\epsilon)}.
\]
\end{theorem}

\begin{proof}
For predictive states separated by more than $2\delta$, the corresponding $\ell_\infty$ balls of radius $\epsilon$ must be disjoint; otherwise one hidden point would decode within $\delta$ of both.  Packing these balls using their centers in $[-R,R]^d$ yields the first bound; the second is its rearrangement.
\end{proof}

For a regular $k$-dimensional predictive manifold with covering number proportional to $\delta^{-k}$, Theorem~\ref{thm:packing} yields $d=\Omega(k\log(1/\delta))$ at fixed normalized noise.  This establishes that continuous memory retention cannot rely on topological basins, but must instead be managed via metric margins---motivating the operational auditing framework we construct next.
\section{Auditable Certification and Adaptive Refinement}\label{sec:audit}

While the complete predictive map $\Phi_\infty$ comprises infinitely many future test coordinates, Whitney-type embedding theory ensures that a finite coordinate chart captures the quotient on compact domains:

\begin{proposition}[Finite predictive embedding]\label{prop:embedding}
Let $\State$ be a compact $C^1$ manifold and suppose every $\phi_w$ is $C^1$.  If future tests separate points ($s\ne s'\Longrightarrow \phi_w(s)\ne\phi_w(s')$) and separate nonzero tangent vectors ($v\in T_s\State\setminus\{0\}\Longrightarrow d\phi_w(s)[v]\ne0$), then a finite probe set $W\subset\Tests$ exists for which $\Phi_W=(\phi_w)_{w\in W}$ is a $C^1$ embedding.
\end{proposition}

The proof is given in Appendix~\ref{app:embedding}.  Proposition~\ref{prop:embedding} naturally grounds predictive state representations \citep{littman2001predictive} within differential observability and embedding theory \citep{takens1981detecting}.  However, moving from topological existence to verifiable runtime safety requires an operational decision metric that can be estimated and audited from finite data, certified under sampling uncertainty, and algorithmically refined when blind spots appear.

\subsection{Audited Semantics and Deterministic Certificates}

To make safety verifiable, we declare a correction domain $\Domain\subseteq\State$, an audit bank $\Audit=\{u_1,\ldots,u_M\}$, and a compact deployment bank $W\subseteq\Audit$ before measuring margins, with $m=\abs W \ll M$.  Define the empirical future metrics $d_\Audit(s,s')=\norm{\Phi_\Audit(s)-\Phi_\Audit(s')}_\infty$ and $d_W(s,s')=\norm{\Phi_W(s)-\Phi_W(s')}_\infty$.  The audited separation margin is defined as
\begin{equation}\label{eq:omega}
 \Omega^{\Domain}_{W\mid\Audit}(\delta)
 =\inf_{\substack{s,s'\in\Domain\\d_\Audit(s,s')\ge\delta}}d_W(s,s').
\end{equation}
By standard convention, the infimum over an empty set is $+\infty$; non-trivial certification assumes that $\Domain$ has audit diameter at least $\delta$.  Crucially, \eqref{eq:omega} operates entirely on observable distances under $\Audit$ rather than requiring unobservable quantities under $d_\infty$, effectively decoupling low operational inference cost ($m=\abs W$) from comprehensive safety verification scope ($M=\abs\Audit$).

When an all-future guarantee is required, it relies on an independently established \emph{completeness modulus} $\rho_\Audit$ satisfying $d_\infty(s,s')\le\rho_\Audit(d_\Audit(s,s'))$ with $\rho_\Audit(r)\to0$ as $r\to0$.  Under uniform continuity and an equicontinuous test family, such a modulus follows directly by compactness; in finite-rank linear systems, core tests provide an analytical special case whose stability is governed by reconstruction conditioning.

\begin{theorem}[Deterministic audit-safety certificate]\label{thm:deterministic}
Let $\psi_W(h)$ estimate $\Phi_W(\pi(h))$ with uniform calibration error at most $\tau$ on $\pi^{-1}(\Domain)$.  If a corrector $Q$ maps $\pi^{-1}(\Domain)$ into itself and obeys $\norm{\psi_W(Q(h))-\psi_W(h)}_\infty\le\xi$, then
\[
 \xi+2\tau<\Omega^{\Domain}_{W\mid\Audit}(\delta)
 \implies
 d_\Audit(\pi(Q(h)),\pi(h))<\delta.
\]
Under completeness modulus $\rho_\Audit$, the corresponding full-future distortion satisfies $d_\infty(\pi(Q(h)),\pi(h))<\rho_\Audit(\delta)$.
\end{theorem}

\begin{proof}
The triangle inequality bounds $d_W(\pi(Q(h)),\pi(h))$ by $\xi+2\tau$.  Both predictive states lie in $\Domain$.  If their audit distance were at least $\delta$, definition \eqref{eq:omega} would give the contradictory lower bound $\Omega^{\Domain}_{W\mid\Audit}(\delta)$.
\end{proof}

\subsection{Finite-Rollout PAC Certification and Sample Complexity}

While Theorem~\ref{thm:deterministic} provides an exact safety condition, deploying it in practice requires estimating predictive distances from finite stochastic rollouts without oracle knowledge of true conditional expectations.

\begin{assumption}[Pre-declared generative audit access]\label{assump:generative}
Before observing any rollout outcomes, the protocol fixes $(\Domain,\Audit,\delta,\eta,N,M,R,\alpha)$ and anchors $s_1,\ldots,s_N\in\Domain$ forming an $\eta$-net of $\Domain$ under $d_\Audit$.  A conditional reset oracle initializes rollouts at $s_i$.  For each anchor $i$, probe $u$, and repetition $r$, it returns a bounded outcome $Y_{i,u}^{(r)}\in[0,1]$ with conditional mean $\phi_u(s_i)$, mutually independent across $(i,u,r)$.  Theorem~\ref{thm:pac} applies to any fixed $W\subseteq\Audit$; Corollary~\ref{cor:cegis-reuse} handles its finite-bank adaptive selection.
\end{assumption}

Let $\widehat\phi_u(s_i)=R^{-1}\sum_{r=1}^R Y_{i,u}^{(r)}$ and let $e_R=\sqrt{\log(2NM/\alpha)/(2R)}$.  Define the empirical active pair set $\widehat{\mathcal A}_\delta=\{(i,j):\widehat d_\Audit(i,j)\ge\delta-2\eta-2e_R\}$, and the empirical margin $\widehat\Omega_{W\mid\Audit}(\delta)=\min_{(i,j)\in\widehat{\mathcal A}_\delta}\widehat d_W(i,j)$.

\begin{theorem}[Finite-audit PAC separation]\label{thm:pac}
Under Assumption~\ref{assump:generative}, with probability at least $1-\alpha$,
\begin{equation}\label{eq:pac}
 \Omega^{\Domain}_{W\mid\Audit}(\delta)\ge
 \underline\Omega^{\Domain}_{W\mid\Audit}(\delta):=
 \widehat\Omega_{W\mid\Audit}(\delta)-2e_R-2\eta.
\end{equation}
Replacing $\Omega^{\Domain}_{W\mid\Audit}$ by its lower certificate in Theorem~\ref{thm:deterministic} yields a PAC-safe corrector.
\end{theorem}

Crucially, $\widehat{\mathcal A}_\delta$ is constructed entirely from observable empirical distances.  The resulting guarantee is rigorously established with respect to the declared audit bank, domain, and threshold; complete proofs appear in Appendix~\ref{app:proofs}.

\begin{corollary}[Joint calibration and rollout guarantee]\label{cor:joint}
Suppose the calibration statement in Theorem~\ref{thm:deterministic} holds with probability at least $1-\beta$ on validation data independent of the audit rollouts.  Then the safety conclusion obtained via \eqref{eq:pac} holds jointly with probability at least $1-\alpha-\beta$.
\end{corollary}

To determine the data requirement for non-trivial certification, suppose the audit metric on $\Domain$ has covering number $N(\Domain,d_\Audit,\eta)\le C_\State\eta^{-k}$ and a strict margin $\Omega^{\Domain}_{W\mid\Audit}(\delta/2)\ge\Omega>0$ exists.

\begin{theorem}[Audited rollout complexity]\label{thm:complexity}
Taking $\eta\le q/8$ and $e_R\le q/16$ for $q=\min\{\delta,\Omega\}$ ensures $\underline\Omega^{\Domain}_{W\mid\Audit}(\delta)\ge\Omega/2$.  Under this covering bound, the total number of individual probe outcomes obeys
\[
 B=\widetilde O\!\left(C_\State Mq^{-(k+2)}\right),
\]
which simplifies to $\widetilde O(C_\State M\Omega^{-(k+2)})$ in the small-margin regime $\Omega\le\delta$.
\end{theorem}

The exponent decomposes naturally: $k$ powers cover the intrinsic predictive state space, while two powers estimate bounded scalar means.  The linear factor $M=\abs\Audit$ reflects the verification scope against which the compact deployment bank is audited.  Importantly, this sample complexity is not an artifact of loose analysis, but fundamental to distribution-free separation testing:

\begin{theorem}[Minimax lower bound]\label{thm:lower}
Fix $\delta_0=1/4$.  For every $k\ge1$ and $L>0$, constants $c_k,\Omega_0>0$ exist such that, for $0<\Omega\le\Omega_0$, any adaptive algorithm distinguishing between $H_0: \Omega_\Phi(\delta_0)\ge2\Omega$ and $H_1: \Omega_\Phi(\delta_0)=0$ with error at most $1/3$ over $L$-Lipschitz maps $\Phi:[0,1]^k\to[0,1]^k$ requires $B\ge c_kL^k\Omega^{-(k+2)}$ Bernoulli coordinate observations.
\end{theorem}

Here $\Omega_\Phi(\delta)=\inf_{\norm{s-s'}_\infty\ge\delta}\norm{\Phi(s)-\Phi(s')}_\infty$.  Appendix~\ref{app:proofs} provides the localized bump construction and change-of-measure proof.

\subsection{Closed-Loop Interface Synthesis via PF-CEGIS}\label{sec:cegis}

Theorem~\ref{thm:pac} guarantees safety for any fixed deployment bank $W\subseteq\Audit$.  However, if an initial $W$ lacks distinguishing future probes, predictive states will alias together ($d_W(s,s')\approx 0$), driving the certified margin $\underline\Omega^{\Domain}_{W\mid\Audit}$ negative and blocking corrector authorization.  To resolve this, audited Predictive-Fiber CEGIS (PF-CEGIS) dynamically refines $W$ guided by certified counterexamples.  Starting from $W_0\subseteq\Audit$, at iteration $t$:

\begin{enumerate}
\item estimate audit and deployment distances with simultaneous confidence intervals;
\item find an anchor pair certified far in $d_\Audit$ but insufficiently separated under $W_t$;
\item select the probe $u^*\in\Audit\setminus W_t$ that maximizes separation for this alias pair;
\item update $W_{t+1}=W_t\cup\{u^*\}$ and recompute the certified lower margin $\underline\Omega^{\Domain}_{W_{t+1}\mid\Audit}(\delta)$;
\item permit correction only when the learned corrector's measured distortion satisfies $\xi+2\tau<\underline\Omega^{\Domain}_{W_{t+1}\mid\Audit}(\delta)$.
\end{enumerate}

Because concentration holds uniformly across all subsets of a fixed, pre-declared audit bank $\Audit$ (Corollary~\ref{cor:cegis-reuse}), selecting successive deployment subsets requires zero additional rollouts.  At the population level, the procedure terminates in at most $M$ steps with $W=\Audit$ as the canonical terminal representation.  In summary, PF-CEGIS provides a principled, counterexample-guided mechanism to synthesize a compact, certifiable deployment interface with provable safety guarantees against the declared audit family.

\section{Empirical Validation}\label{sec:experiments}

All experiments are fully reproducible and emphasize direct quantification of semantic safety alongside task performance.  The controlled synthetic studies are generated by \texttt{experiments/run\_all.py}; the MuJoCo physical generative certificate, UCI HAR audit, and digit stream by \texttt{run\_mujoco\_generative\_pac.py}, \texttt{run\_har\_modern.py}, and \texttt{run\_sequential\_digits.py}.  Every stochastic rollout experiment records literal Bernoulli probe outcomes while leveraging binomial sufficient statistics for exact, efficient simulation.  Table~\ref{tab:main-metrics} summarizes the complete suite of results under our safety-first evaluation paradigm.

\begin{table*}[t]
\centering
\footnotesize
\setlength{\tabcolsep}{2pt}
\caption{Comprehensive evaluation under the safety-first paradigm.  Nonanalytic margins are evaluated relative to their pre-declared correction domains.  ``Known quotient'' studies feature analytic ground-truth semantics.  The HAR and digit benchmarks are reported as held-out empirical audits, distinguishing passive observational datasets from resettable generative PAC certificates.}
\label{tab:main-metrics}
\begin{tabular}{lcccccc}
\toprule
Study & task error & $\underline\Omega_{W\mid\Audit}$ & $\tau$ & $\xi$ & cert. distortion & normal contraction \\
\midrule
DFA hard snap & $\DFAHardError$ & analytic & $0$ & $0$ & $0$ (basin) & $\DFAHardContraction$ \\
Analog fiber & $\FiberSemanticRMSE$ & analytic & $0$ & $0$ & $0$ & $1.000$ \\
Mixed fiber & $z$ $\MixedFiberZRMSE$ & analytic & $0$ & $0$ & $0$ & $1.000$ \\
Rollout certificate & --- & $\MinimumCertificate$ & $\CertificationTau$ & $0$ & $<\delta$ if gated & --- \\
PF-CEGIS final & --- & $\RefinedCEGISCertificate$ & $\CEGISTau$ & $\CEGISXi$ & $\le\delta$ & $1.000$ \\
Re-anchoring & law $\ReanchoringLawError$ & --- & $0$ & $0$ & observation & --- \\
MuJoCo generative PAC & gate pass & $\MujocoCertificate$ & $0.141$--$0.168$ & $\MujocoPFXi$ & $<0.80$ & $\MujocoNormalContraction$ \\
MuJoCo double pendulum & gate pass & $\DoubleCertificate$ & $0.135$--$0.183$ & numerical & $<0.80$ & $0.220$ \\
HAR modern PF (held-out) & GRU $\HARModernGRUError$ & passive & passive & $\HARModernPFXi$ & empirical & $\HARModernContraction$ \\
Digits PF (held-out) & $\DigitsPFError\pm\DigitsPFErrorStd$ & $\widehat\Omega=\DigitsEnrichedMargin$ & passive & $0$ & $\DigitsPFInterface$ & $\DigitsPFContraction$ \\
\bottomrule
\end{tabular}
\end{table*}

\begin{figure*}[t]
 \centering
 \includegraphics[width=0.92\textwidth]{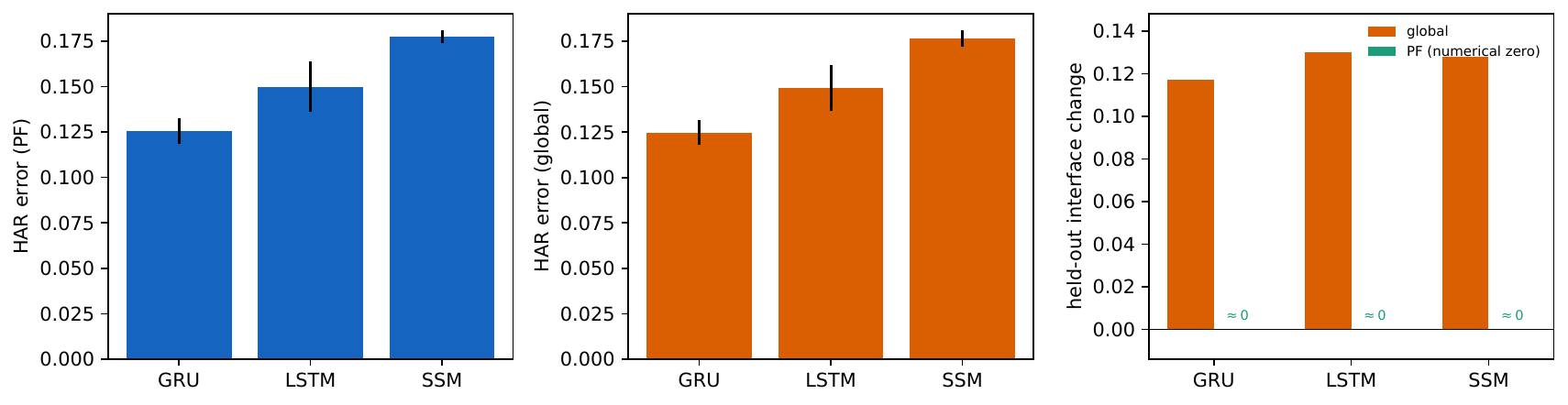}
\caption{Subject-disjoint UCI HAR evaluation under hidden-state noise (mean $\pm$ s.d., \HARModernSeeds\ seeds per architecture).  The third panel deliberately draws global and PF values side-by-side: PF interface change is numerical zero ($\approx10^{-15}$), whereas global contraction changes every declared head.}
 \label{fig:har}
\end{figure*}

\subsection{Conceptual Geometry: Discrete Snapping vs.\ Continuous Preservation}

We first validate the geometric foundations established in Section~\ref{sec:geometry} across three controlled regimes: finite discrete codes, continuous analog memory, and hybrid state spaces.

In discrete state spaces, finite predictive equivalence classes permit robust basin snapping.
In an eight-state DFA with deterministic transitions, we inject recurrent hidden noise of $\sigma=0.23$ per step.  As shown in Figure~\ref{fig:dfa}, nearest-state hard snapping achieves low state error ($\DFAHardError$), dramatically outperforming both uncorrected execution ($\DFAUncorrectedError$) and a noise-trained soft recurrent baseline ($\DFASoftError$).  This confirms that discrete predictive codes admit isolated, positive-radius basins of attraction where aggressive snapping successfully restores exact ground truth.

\begin{figure}[t]
 \centering
 \includegraphics[width=\linewidth]{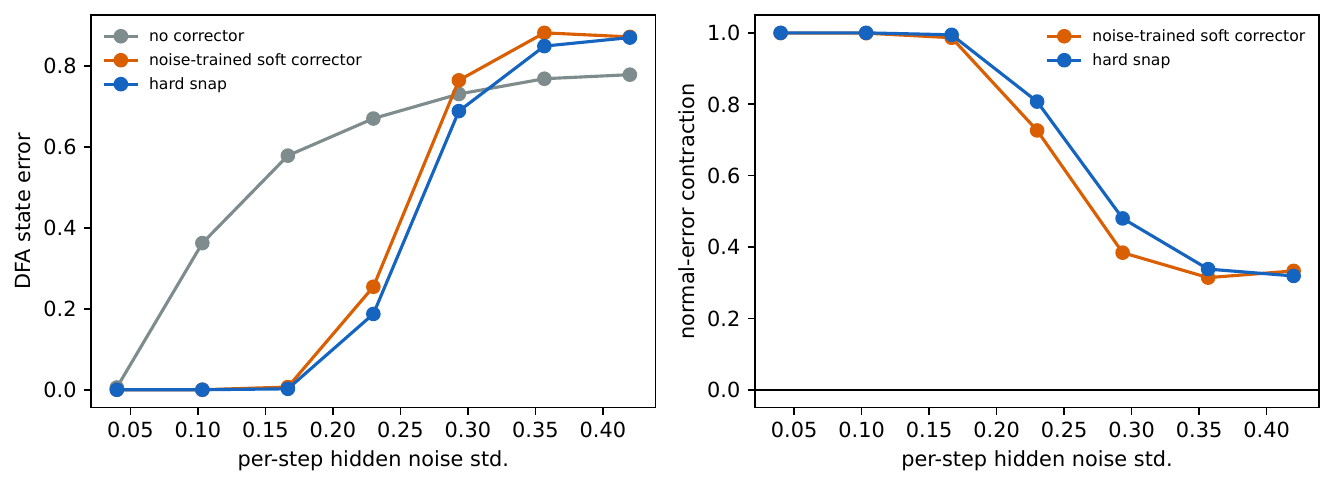}
 \caption{Finite predictive states permit strong discrete correction.  Hard snapping is effective inside its basin; the learned soft corrector is a noise-trained recurrent baseline.}
 \label{fig:dfa}
\end{figure}

Unlike discrete automata, continuous analog memory manifolds require tangent preservation to avoid semantic drift.
To demonstrate the danger of global contraction, we consider an analytic one-dimensional integrator realization $h=(z,r_1,r_2)$, where $z$ represents continuous memory and $r_1,r_2$ are deliberately injected nonpredictive fiber noise.  The fiber corrector applies a local multiplier spectrum of $(1,0,0)$, projecting noise along the fiber, whereas global contraction applies isotropic decay $(0.78, 0.78, 0.78)$.  Figure~\ref{fig:analog} illustrates the central tangent--normal distinction: while both methods suppress fiber noise, global contraction contracts the memory coordinate, increasing semantic RMSE from $\FiberSemanticRMSE$ to $\GlobalSemanticRMSE$.  In contrast, the fiber corrector drives fiber-noise RMSE to $\FiberNormalRMSE$ with zero semantic bias.

\begin{figure}[t]
 \centering
 \includegraphics[width=\linewidth]{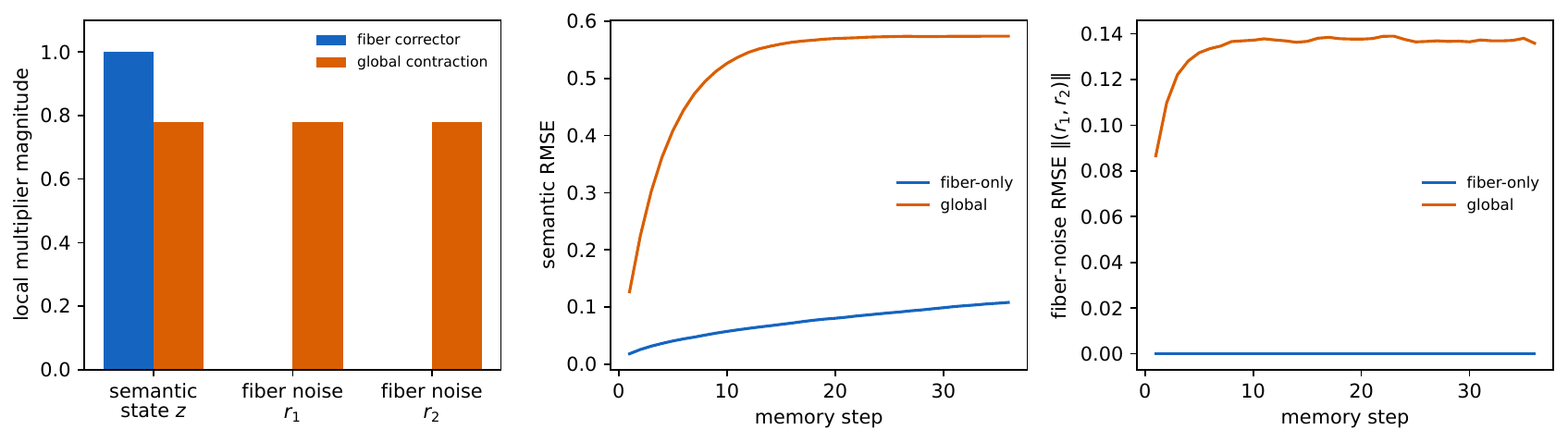}
 \caption{Continuous integrator realization $h=(z,r_1,r_2)$.  The two $r$ coordinates are intentionally nonpredictive fiber noise, not additional memories.  Fiber-only correction preserves $z$ and contracts $(r_1,r_2)$; global contraction contracts all three coordinates and therefore biases analog memory.}
 \label{fig:analog}
\end{figure}

In hybrid discrete--continuous architectures, the state decomposes into distinct topological components.
In a mixed state $s=(q,z)$ comprising a 4-valued categorical mode $q$ and a continuous coordinate $z$, the optimal corrector snaps the one-hot mode, preserves $z$, and eliminates two transverse noise coordinates.  Figure~\ref{fig:mixed} demonstrates exact mode recovery alongside $z$ RMSE of $\MixedFiberZRMSE$, whereas global contraction yields $\MixedGlobalZRMSE$ and unnecessary $z$ quantization yields $\MixedQuantizedZRMSE$.  Together, these experiments confirm that strong contraction is semantically sound across discrete modes, but must strictly vanish along continuous predictive coordinates.

\begin{figure}[t]
 \centering
 \includegraphics[width=\linewidth]{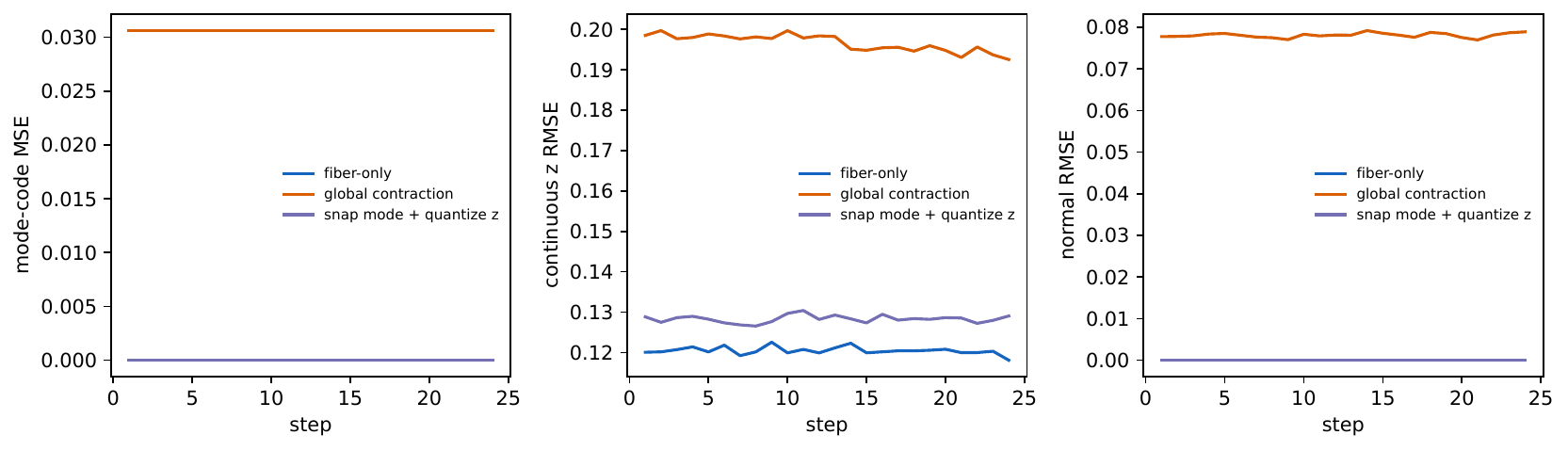}
 \caption{Hybrid predictive state $s=(q,z)$.  The mode is safely snapped, whereas the continuous coordinate must remain a semantic degree of freedom.}
 \label{fig:mixed}
\end{figure}

\subsection{Statistical Certification and Closed-Loop Refinement}

We next evaluate whether the finite-rollout PAC certificate reliably authorizes safe state corrections, and whether PF-CEGIS dynamically repairs uncertified deployment banks.

Evaluating statistical certificate scaling under finite sampling confirms our theoretical sample complexity bounds.
On the synthetic domain $s\in[0,1]^k$, the audit bank contains strong coordinates $0.1+0.8s_j$ and the deployment bank contains weak coordinates $0.5+(0.8\Omega/\delta)(s_j-1/2)$, so the ground-truth margin equals $\Omega$.  Varying rollout budget $B$, intrinsic dimension $k$, and true margin $\Omega$ with independent Bernoulli rollouts ($\alpha=0.05$), Figure~\ref{fig:scaling} demonstrates that the certified lower bound is negative at underfunded budgets ($\LowBudgetCertificate$) and turns positive once the sufficient schedule is reached ($\HighBudgetCertificate$).  The empirical sample complexity exponents are $\ScalingExponentOne$ for $k=1$ and $\ScalingExponentTwo$ for $k=2$, closely matching the theoretical $k+2$ minimax rate of Theorem~\ref{thm:lower}.  Offline schedules, wall times, and the curse of predictive dimension are detailed in Appendix~\ref{app:efficiency} (Figure~\ref{fig:efficiency}).

\begin{figure}[t]
 \centering
 \includegraphics[width=\linewidth]{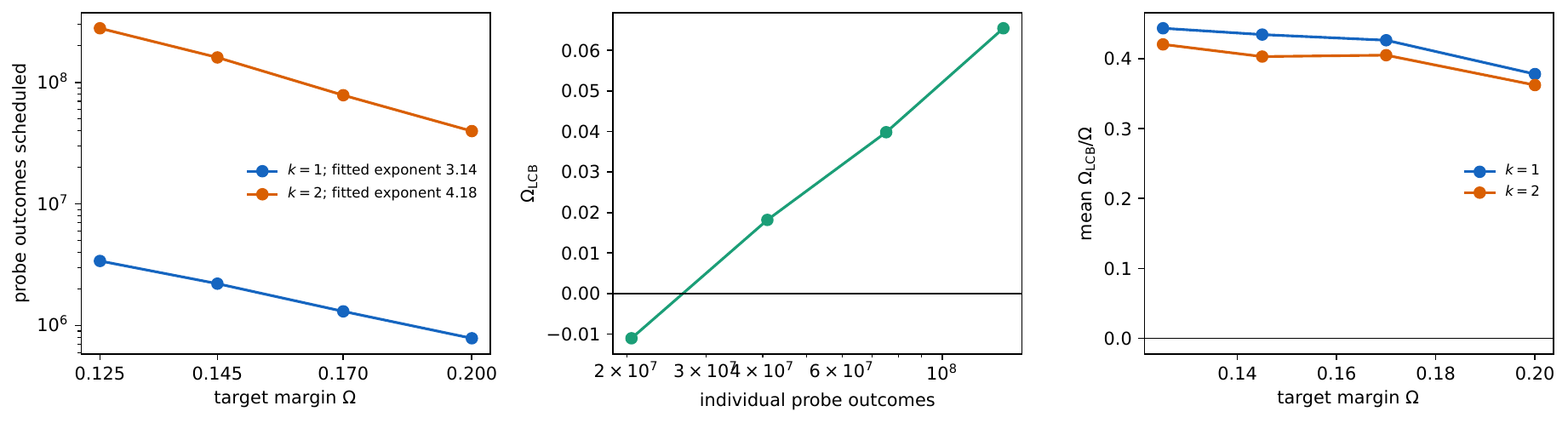}
 \caption{Finite-audit certification.  We vary margin, intrinsic dimension, and literal rollout budget, and report the lower certificate rather than only prediction accuracy.}
 \label{fig:scaling}
\end{figure}

To synthesize certifiable interfaces without manual trial-and-error, adaptive refinement via PF-CEGIS iteratively queries violating states.
Starting from a constant probe where all predictive directions alias together, PF-CEGIS iteratively detects distinguishing failures in $\Audit$ and incorporates the most separating probe into $W$.  Figure~\ref{fig:cegis} tracks this trajectory: the certified margin rises monotonically from $\InitialCEGISCertificate$ to $\RefinedCEGISCertificate$, while corrector semantic distortion falls from $\InitialAuditDistortion$ to $\FinalAuditDistortion$.  Crucially, the safety gate blocks correction until the final round when $\xi+2\tau<\underline\Omega$ is rigorously established.  Because audit rollouts are pre-declared, probe selection incurs zero additional simulation cost (Corollary~\ref{cor:cegis-reuse}).

\begin{figure}[!htbp]
 \centering
 \includegraphics[width=0.86\linewidth]{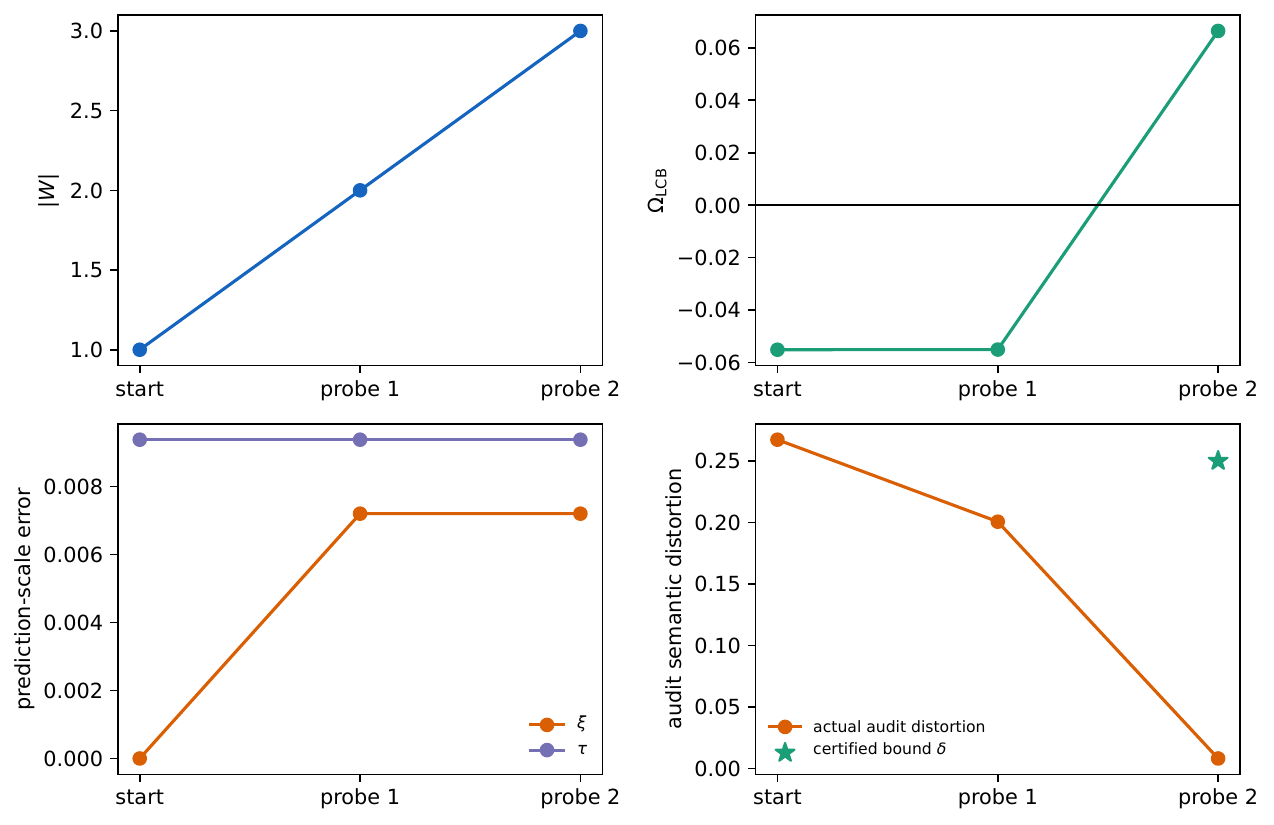}
 \caption{PF-CEGIS tracks the entire safety gate, not just its final probe count.  Adding probes makes the certificate stronger and narrows the information the corrector is allowed to delete.}
 \label{fig:cegis}
\end{figure}

Over extended operational horizons, uncorrected analog drift accumulates along neutral directions unless periodically re-anchored by sensory observations.
Fiber correction eliminates internal representational noise but cannot, by itself, eliminate genuine semantic drift along marginal memory directions.  To formalize this boundary, consider an error recurrence $e_{t+1}\le Le_t+b$, where $b$ accounts for unmodeled disturbance and certified correction distortion.  If $L<1$, internal contraction stabilizes memory with asymptotic error bounded by $b/(1-L)$.  However, for marginal continuous memory where $L=1$, persistent semantic drift is not removable by any internal semantics-preserving operation.  In such neutral memory regimes, stabilization requires external evidence: if informative observations re-anchor the state every $H$ steps with contraction factor $\kappa<1$, the pre-anchor error satisfies
\begin{equation}\label{eq:reanchor}
 \sup_t e_t \le \frac{Hb}{1-\kappa} + O(b).
\end{equation}
Figure~\ref{fig:reanchor} simulates this exact scalar recurrence across varying disturbance scales $b$ and observation intervals $H$, overlaying empirical steady-state errors with the theoretical bound \eqref{eq:reanchor}.  The maximum relative discrepancy is \ReanchoringLawError, demonstrating that internal fiber correction (which eliminates transversal representational noise) and sensory re-anchoring (which bounds longitudinal drift) operate as complementary stabilizing mechanisms.

\begin{figure}[!htbp]
 \centering
 \includegraphics[width=0.74\linewidth]{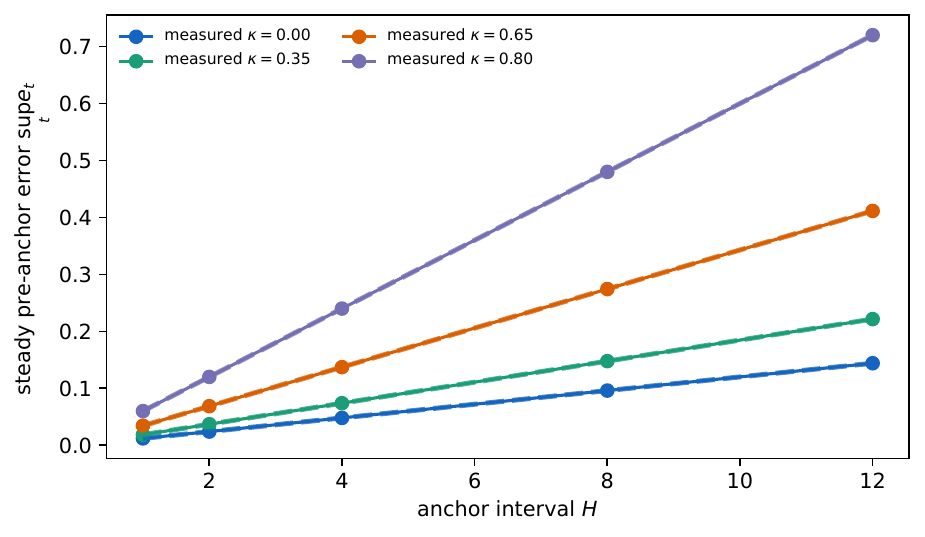}
\caption{Periodic informative observations re-anchor neutral continuous memory.  Solid lines are simulated steady pre-anchor errors; dashed lines are the theoretical $Hb/(1-\kappa)$ predictions.}
\label{fig:reanchor}
\end{figure}

\subsection{Physical Dynamics and Real-World Observational Streams}

Finally, we test the complete framework on physical continuous dynamics under exact resets and real sequential benchmarks under observational audits, demonstrating both active PAC certification and passive empirical verification.

Turning to continuous physical dynamics, we evaluate our generative certificate on the MuJoCo InvertedPendulum-v4 domain.
In the MuJoCo InvertedPendulum environment, we define an audit domain across a 75-anchor grid covering the continuous phase cylinder $(\theta, \dot\theta) \in [-\pi, \pi] \times [-8, 8]$ with exact physical state resets.  GRU, LSTM, and an input-selective SSM learn 9 randomized-control terminal events that evaluate whether the pole remains upright within finite horizons under stochastic control torques.  With $\delta=0.80$, $\alpha=0.05$, and $B=\MujocoBudget$ literal outcomes, fresh physical resets yield a certified lower margin $\underline\Omega=\MujocoCertificate$.  Disjoint calibration rollouts rigorously bound predictor errors $(\tau_{\rm GRU},\tau_{\rm LSTM},\tau_{\rm SSM})=(\MujocoGRUTau,\MujocoLSTMTau,\MujocoSSMTau)$, allowing all three distinct recurrent architectures to comfortably pass the safety gate $\xi+2\tau<\underline\Omega$.  Across all models, PF achieves minimal interface distortion $\xi\le\MujocoPFXi$ while aggressively contracting normal directions by $\MujocoNormalContraction$ (in sharp contrast to $\xi\ge\MujocoGlobalXi$ for isotropic global contraction).  Furthermore, physical PF-CEGIS successfully expands an uncertified 3-probe interface ($\MujocoCEGISInitial$) into a certified 9-probe bank ($\MujocoCEGISFinal$), and strictly positive safety margins are maintained on the more challenging 4D nonlinear InvertedDoublePendulum dynamics.  Detailed rollout distributions and calibration bounds confirm robust interface preservation across diverse dynamical regimes (Figure~\ref{fig:resettable-detail}).

\begin{figure}[t]
\centering
\includegraphics[width=\columnwidth]{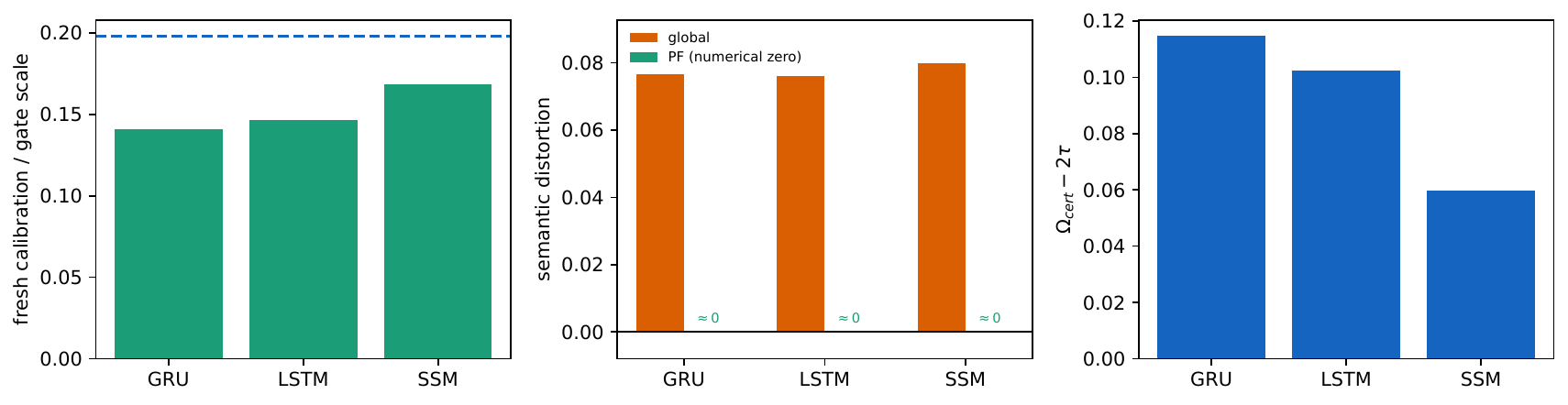}
\caption{MuJoCo physical generative certificate.  GRU, LSTM, and a selective SSM are trained on randomized-control rigid-body rollouts; disjoint exact-reset rollouts bound calibration and certify audit separation.  The middle panel compares global and PF distortion: PF achieves numerical zero distortion, whereas isotropic global contraction corrupts the declared interface.}
\label{fig:resettable-detail}
\end{figure}

Beyond active simulation resets, we evaluate observational audit preservation on official subject-disjoint splits of the UCI Human Activity Recognition (HAR) dataset.
Evaluating across official subject-disjoint splits of 128-step inertial sequences (Figure~\ref{fig:har}), GRU, LSTM, and SSM models were trained from scratch over \HARModernSeeds\ seeds.  The 6-class activity taxonomy (walking, walking upstairs, walking downstairs, sitting, standing, laying) tests whether predictive fibers generalize across anatomical and stylistic movement variations among unseen test subjects.  Under hidden-state noise ($\sigma=0.12$), PF preserves all declared 6-way linear activity heads to numerical precision ($\approx 10^{-15}$, $\xi=\HARModernPFXi$) while contracting the normal complement by $\HARModernContraction$.  In contrast, global contraction corrupts every declared output by at least $\HARModernGlobalXi$, degrading margin boundaries between static postures and dynamic gaits.  This held-out empirical audit confirms consistent fiber preservation on real human activity time series.

To examine decision-making under severe partial observability, we evaluate a finite-horizon POMDP that reveals handwritten digits pixel-by-pixel.
We evaluate a finite-horizon POMDP by serially revealing $8\times8$ optical handwritten digits pixel-by-pixel over 64 clocked steps.  Using a clocked 64-dimensional linear memory RNN, this structured setup isolates correction geometry from sequence optimization failures: early observations contain ambiguous pixel fragments, requiring the hidden state to integrate ambiguous evidence until distinctive structural strokes (e.g., loops in 0, 6, 8 or horizontal bars in 4, 5, 7) emerge.  With independent per-update noise ($\sigma=0.025$), PF achieves digit error $\DigitsPFError\pm\DigitsPFErrorStd$, significantly outperforming isotropic global contraction ($\DigitsGlobalError\pm\DigitsGlobalErrorStd$) and stable contractive recurrent baselines ($\DigitsSpectralError$).  It preserves the 14-head predictive interface (class probabilities and future pixel projections) to $\DigitsPFInterface$ while contracting normal directions by $\DigitsPFContraction$, demonstrating reliable semantic protection under severe partial observability (Figure~\ref{fig:digits-detail}).

\begin{figure}[t]
\centering
\includegraphics[width=\columnwidth]{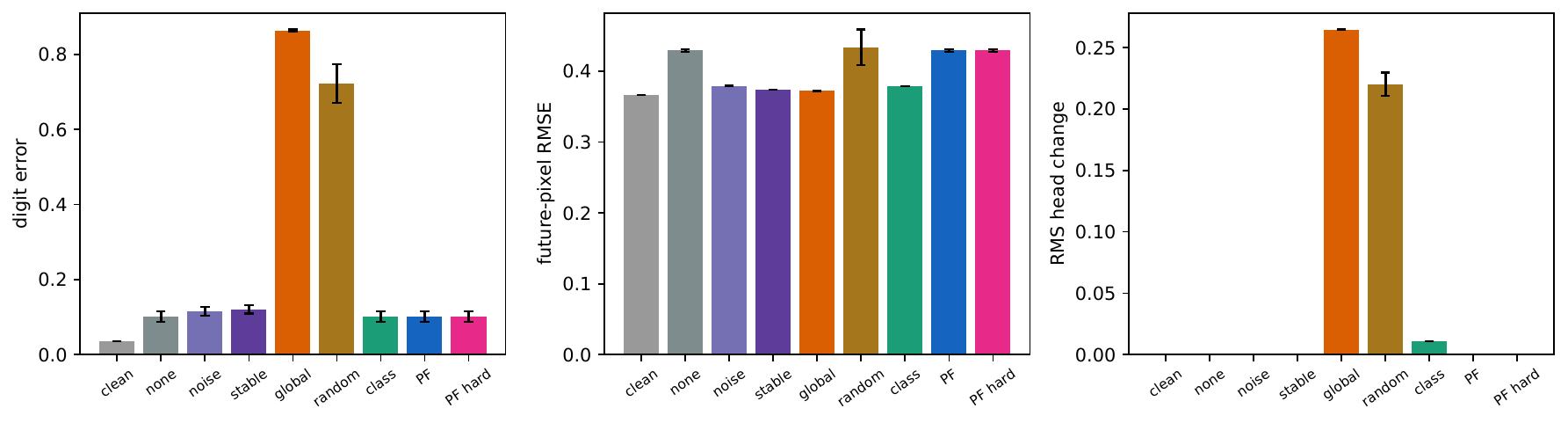}
\caption{Real raster-revealed digit stream under partial observability.  PF contracts normal directions while strictly preserving the declared class-and-future interface; global contraction corrupts class boundaries.}
\label{fig:digits-detail}
\end{figure}

\section{Discussion and Related Work}\label{sec:discussion}

The operational distinction between active generative access and passive observational audits is fundamental to empirical certification.
The PAC certification framework in Theorem~\ref{thm:pac} leverages resettable, independent probe access.  A natural question is whether one could dispense with resettable environments and establish distribution-free certificates from passive trajectory logs alone.  The following structural result demonstrates that active coverage is fundamentally indispensable:

\begin{proposition}[Coverage is necessary for global certification]\label{prop:coverage}
If a data-collection policy never visits a nonempty predictive region $A\subset\State$, then no estimator based exclusively on its trajectory rollouts can yield a valid distribution-free audit-separation certificate over all of $\State$.
\end{proposition}
\begin{proof}
Construct two dynamical environments that agree identically on $\State\setminus A$ but induce distinct future conditional expectations on $A$.  They yield identical collected-data distributions under the given policy, yet possess different true separation margins on $\State$.  Hence, no purely observational statistic can separate them with uniform confidence.
\end{proof}

This structural barrier delineates the operational boundary between our resettable PAC guarantees (such as the MuJoCo rigid-body certificates) and passive observational evaluations (such as the UCI HAR and digit-stream benchmarks).  In offline observational settings, unvisited or rarely visited predictive regions harbor latent epistemic uncertainty that cannot be certified away without parametric or mixing assumptions.  When dependent sequential data are available, a standard $\beta$-mixing blocking argument replaces the literal rollout count $R$ by an effective number of approximately independent temporal blocks \citep{yu1994rates}, while lower bounds on anchor stationary occupancy are required to govern non-uniform sample complexity \citep{grinberg2018nonuniform}.

Our geometric framework bridges several foundational perspectives in representation learning, dynamical systems, and predictive-state representations.
Our framework bridges several foundational perspectives in representation learning and dynamical systems.  While predictive-state representations construct recurrent states from future statistics \citep{littman2001predictive,downey2017psrnn,venkatraman2017psd}, our work addresses the orthogonal problem of determining which internal state perturbations preserve those statistics under post-hoc correction.  Similarly, while stable recurrent architectures enforce contractive state dynamics \citep{miller2019stable}, our geometric analysis identifies the exact quotient directions where contraction is semantically destructive.  From an information-theoretic viewpoint, predictive rate--distortion bounds optimal compression for future prediction \citep{marzen2016rd,hahn2019nprd}; in contrast, our audited margin provides an operational decision boundary authorizing valid state corrections.  Finally, while continuous-attractor models explore tangent and transverse memory dynamics \citep{sagodi2024continuous} and empirical observability quantifies state distinguishability \citep{massiani2024observability}, we synthesize these insights into a certifiable, finite-sample verification framework.  By formalizing predictive fibers, we establish an auditable boundary that decouples internal error rejection from semantic memory preservation.

In practical real-world deployments, our framework provides modular, auditable safety guarantees.
In practical deployments, our audit framework provides transparent, modular guarantees: any certified claim clearly delineates its declared audit family, anchor coverage, predictor calibration $\tau$, and corrector distortion $\xi$.  By making these structural assumptions explicit, the framework ensures that safety verification is auditable, robust, and interpretable.  When an autonomous system operates in an uncertified state region, the framework refrains from making unfounded safety assertions, flagging the state for fallback or sensory re-anchoring rather than risking silent semantic corruption.

\section{Conclusion}\label{sec:conclusion}

The predictive quotient characterizes the fundamental geometric limit of what a recurrent representation may safely forget.  It establishes that contraction is permissible within predictive fibers, but strictly forbidden along future-distinguishing coordinates.  We demonstrate that finite stochastic rollouts provide statistically certified separation margins relative to a declared audit family, with sample complexity matching optimal minimax rates under Lipschitz regularity.  The foundational principle is clear:
\emph{A recurrent model may safely forget only what cannot change its future.}

\section*{Impact Statement}
This work advances the reliability and interpretability of recurrent models by providing a rigorous, auditable methodology for internal state correction and compression.  By formalizing predictive fibers and decoupling deployment inference from audit verification, the framework offers transparent safety criteria for autonomous systems, robotics, and sequential decision-making.  We emphasize that responsible deployment of these certificates requires careful validation of operational assumptions---including probe coverage, calibration bounds, and domain alignment.  Distinguishing between active generative certificates and passive empirical audits fosters robust, accountable evaluation practices in safety-critical machine learning.

\FloatBarrier\clearpage
\bibliographystyle{icml2026}
\bibliography{references}

\clearpage
\FloatBarrier\clearpage
\appendix
\renewcommand{\thetable}{A\arabic{table}}
\renewcommand{\thefigure}{A\arabic{figure}}
\setcounter{table}{0}
\setcounter{figure}{0}

\section{Offline Certification Cost and Proof Details}\label{app:efficiency}

\begin{table*}[t]
\centering
\scriptsize
\caption{Proof-assumption audit.  The final column characterizes structural failure modes and counterexample classes when premises are violated.}
\label{tab:proof-audit}
\begin{tabular}{p{0.20\textwidth}p{0.34\textwidth}p{0.34\textwidth}}
\toprule
Claim & Premise used in the proof & Failure when absent \\
\midrule
Quotient rank bound & $\pi$ is a $C^1$ rank-$k$ submersion locally and $Q(h)=h$ for the derivative identity. & At a non-fixed point, derivatives live at $h$ and $Q(h)$ and $D\pi_h(DQ_h-I)=0$ need not hold. \\
Finite embedding / all-future lift & Compact state manifold, $C^1$ tests separating points and tangents; additionally $\ell_\infty$ continuity of the full predictive map for a modulus. & Finite coordinate separation alone does not control an unmeasured future or supply a numerical conditioning bound. \\
PAC separation & Pre-declared $(\Domain,\Audit,\delta)$, an $\eta$-net of $\Domain$, conditional reset at each anchor, bounded independent outcomes, and $W\subseteq\Audit$. & Passive or adaptively selected trajectories can miss a region or invalidate Hoeffding/union-bound coverage. \\
Joint safety statement & Uniform calibration holds on the correction domain and is validated independently (or with a joint confidence budget). & A corrector can pass a margin measured with an overfit predictor while changing true deployment semantics. \\
$\Omega^{-(k+2)}$ upper bound & Covering-number bound and strict margin at $\delta/2$; $q=\min\{\delta,\Omega\}$. & A cover without a positive margin cannot yield a positive certificate at any finite sample size. \\
Minimax lower bound & Lipschitz alternatives remain in $[0,1]^k$, bump supports are disjoint, and Bernoulli means stay away from 0 and 1. & The localized alternatives may not be admissible, and the KL comparison no longer establishes the stated rate. \\
PF-CEGIS gate & The finite audit bank is fixed and fully sampled, so the base mean-concentration event is uniform over its subsets; corrector calibration is independent or jointly controlled. & Expanding the audit bank or adaptively validating a corrector without new or simultaneous data loses the nominal confidence level. \\
\bottomrule
\end{tabular}
\end{table*}
Figure~\ref{fig:efficiency} separates the offline costs that the theorem and the implementation impose.  The left panel instantiates the sufficient schedule at fixed $\Omega=0.25$: its fourth-dimensional instance already schedules \EfficiencyKFourBudget\ individual probe outcomes, illustrating the curse of intrinsic predictive dimension rather than a cost of the recurrent forward pass.  The center panel executes exact-reset MuJoCo audits on the 75-anchor physical domain.  Increasing the literal budget from $84{,}375$ to \EfficiencyPhysicalBudget\ outcomes requires \EfficiencyPhysicalLowSeconds--\EfficiencyPhysicalHighSeconds\ seconds on an Apple-silicon workstation, yielding a certified lower margin \EfficiencyPhysicalCertificate\ at the largest budget.  The right panel demonstrates the computational efficiency of reusing a fully sampled, fixed audit bank: selecting each successive PF-CEGIS deployment subset requires zero additional rollout evaluations, with post-audit selection executing in under \EfficiencyCEGISMilliseconds\ ms, precisely reflecting the theoretical reuse guarantee of Corollary~\ref{cor:cegis-reuse}.

\begin{figure}[t]
\centering
\includegraphics[width=\columnwidth]{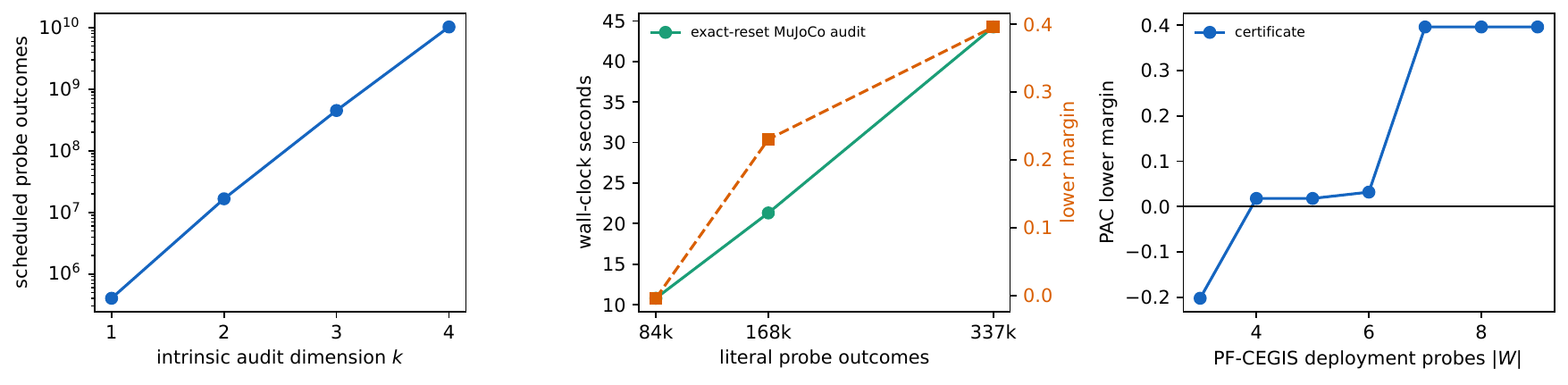}
\caption{Offline certification cost.  Left: the theorem's literal sufficient schedule grows rapidly with intrinsic audit dimension.  Center: exact-reset InvertedPendulum audit wall time and lower margin as rollout budget grows.  Right: PF-CEGIS selects deployment subsets from a fixed fully sampled audit bank with no additional rollout cost.}
\label{fig:efficiency}
\end{figure}

\section{Proofs and technical details}\label{app:proofs}

\subsection{Proof of Proposition~\ref{prop:embedding}}\label{app:embedding}
The unit tangent bundle $T^1\State$ is compact.  The open sets
\[
 \{(s,v)\in T^1\State:d\phi_w(s)[v]\ne0\},\qquad w\in\Tests,
\]
cover it, so a finite $W_1$ makes $d\Phi_{W_1}$ injective on every tangent space.  Thus $\Phi_{W_1}$ is an immersion and locally one-to-one.  A neighborhood $V$ of the diagonal in $\State\times\State$ is consequently separated by $W_1$.  The compact set $(\State\times\State)\setminus V$ is covered by the open point-separation sets $\{(s,s'):\phi_w(s)\ne\phi_w(s')\}$, so finitely many further probes $W_2$ cover it.  Hence $\Phi_{W_1\cup W_2}$ is a continuous one-to-one immersion.  A continuous injection from a compact space to a Hausdorff space is a homeomorphism onto its image, which proves that this map is a $C^1$ embedding.

\subsection{Proof of Theorem~\ref{thm:pac}}
Hoeffding's inequality and a union bound over the $NM$ anchor--probe means give, with probability at least $1-\alpha$,
\[
\max_{i,u}\abs{\widehat\phi_u(s_i)-\phi_u(s_i)}\le e_R.
\]
Condition on this event.  Since $W\subseteq\Audit$, $d_W\le d_\Audit$, so the same $\eta$-net controls both metrics.  Let $s,s'\in\Domain$ satisfy $d_\Audit(s,s')\ge\delta$ and choose nearest anchors $s_i,s_j$.  Triangle inequality gives $d_\Audit(s_i,s_j)\ge\delta-2\eta$.  The empirical distance differs from the true one by at most $2e_R$, so $(i,j)\in\widehat{\mathcal A}_\delta$.  Again using $W\subseteq\Audit$,
\begin{align*}
 d_W(s,s')&\ge d_W(s_i,s_j)-2\eta\\
 &\ge \widehat d_W(i,j)-2e_R-2\eta\\
 &\ge \widehat\Omega_{W\mid\Audit}(\delta)-2e_R-2\eta.
\end{align*}
Taking the infimum over $s,s'$ proves \eqref{eq:pac}.

\subsection{Proof of Corollary~\ref{cor:joint}}
Intersect the calibration event with the Hoeffding event used in Theorem~\ref{thm:pac}, and apply a union bound.

\subsection{Proof of Theorem~\ref{thm:complexity}}
Let $q=\min\{\delta,\Omega\}$.  The covering assumption supplies $N=O(C_\State q^{-k})$ anchors at $\eta\le q/8$.  Choosing $R=O(q^{-2}\log(NM/\alpha))$ makes $e_R\le q/16$.  Any empirically admitted pair has true audit distance at least
\[
 \delta-2\eta-4e_R\ge\delta/2.
\]
Its true $W$ distance is therefore at least $\Omega$, so its empirical distance is at least $\Omega-2e_R$.  Equation~\eqref{eq:pac} is at least $\Omega-4e_R-2\eta\ge\Omega/2$.  Multiplying $N$, $M$, and $R$ proves the display.

\subsection{Proof of Theorem~\ref{thm:lower}}
Write $a=8\Omega$ and set
\[
 \Phi^{(0)}_\ell(s)=\frac12+a\left(s_\ell-\frac12\right),\qquad \ell=1,\ldots,k.
\]
Choose $\Omega_0\le\min\{L/32,1/24\}$.  For $0<\Omega\le\Omega_0$, this map has range in $[1/3,2/3]^k$ and is $L/2$-Lipschitz under $\ell_\infty$ norms.  Moreover,
\[
 \begin{aligned}
 \lVert\Phi^{(0)}(s)-\Phi^{(0)}(s')\rVert_\infty
 &=8\Omega\lVert s-s'\rVert_\infty\\[-2pt]
 &\ge2\Omega\quad\text{if }\lVert s-s'\rVert_\infty\ge\delta_0.
 \end{aligned}
\]
Thus $\Phi^{(0)}\in H_0$.

Put $h=4\Omega/L$.  For a constant depending only on $k$, one can pack
\[
 M_0\ge c'_k(L/\Omega)^k
\]
disjoint $\ell_\infty$ balls $B_\infty(t_j,h)$ in the interior strip
$\{s:\delta_0+h<s_1<1-h,\ h<s_\ell<1-h\ (\ell>1)\}$.  Let
$u_j=t_j-\delta_0e_1$.  Since $h<\delta_0$, $u_j\notin B_\infty(t_j,h)$.  Define the tent bump and the alternative map by
\[
 b_j(s)=2\Omega\left[1-\frac{\lVert s-t_j\rVert_\infty}{h}\right]_+.
\]
\[
 \Phi^{(j)}(s)=\Phi^{(0)}(s)-b_j(s)e_1.
\]
The bump is $L/2$-Lipschitz and has magnitude at most $2\Omega$; consequently $\Phi^{(j)}$ is $L$-Lipschitz and has range in $[1/4,3/4]^k$.  Because $b_j(t_j)=2\Omega$, $b_j(u_j)=0$, and the baseline changes only in coordinate one along this displacement,
\[
 \Phi^{(j)}(t_j)=\Phi^{(j)}(u_j),
 \qquad \norm{t_j-u_j}_\infty=\delta_0.
\]
Hence every $\Phi^{(j)}$ belongs to $H_1$.

Consider any possibly randomized adaptive algorithm, and let $P_0,P_j$ denote its complete transcript laws under $\Phi^{(0)},\Phi^{(j)}$.  Let $N_j$ count queries that request coordinate one at a state in $B_\infty(t_j,h)$.  The balls are disjoint, so $\sum_j\E_0N_j\le B$; for some $j^*$, $\E_0N_{j^*}\le B/M_0$.  The two observation laws agree outside this ball.  Inside, their Bernoulli means differ by at most $2\Omega$ and lie in $[1/4,3/4]$, where
\[
 \mathrm{kl}(\operatorname{Ber}(p)\Vert\operatorname{Ber}(q))
 \le C(p-q)^2.
\]
The adaptive KL chain rule therefore gives
\[
 \mathrm{KL}(P_0\Vert P_{j^*})
 \le 4C\Omega^2\E_0N_{j^*}
 \le4C\Omega^2B/M_0.
\]
If $B<cM_0\Omega^{-2}$ for a sufficiently small universal $c$, Pinsker's inequality gives $\operatorname{TV}(P_0,P_{j^*})<1/3$.  Le Cam's two-point inequality then forces the sum of the two testing errors above $2/3$, contradicting error at most $1/3$ under both $H_0$ and $H_1$.  Substituting the packing lower bound for $M_0$ yields $B\ge c_kL^k\Omega^{-(k+2)}$.

\begin{corollary}[Safe finite-bank PF-CEGIS reuse]\label{cor:cegis-reuse}
Fix $(\Domain,\Audit,\delta)$ before sampling.  On the Hoeffding event in Theorem~\ref{thm:pac}, inequality \eqref{eq:pac} holds simultaneously for every $W\subseteq\Audit$ for which the empirical minimum is defined.  Hence $W$ may be selected adaptively from this fixed, fully sampled audit bank.  If the search expands $\Audit$, changes $\delta$ or $\Domain$, or adaptively trains and validates the corrector, it instead requires fresh data or a corresponding simultaneous confidence statement.
\end{corollary}
\begin{proof}
The event in the proof of Theorem~\ref{thm:pac} bounds every one of the $NM$ pre-declared means.  For each $W\subseteq\Audit$, that proof uses only this same event and $d_W\le d_\Audit$.  Intersecting no additional random events proves simultaneity over the finite family of banks.
\end{proof}

\subsection{Remarks on Proposition~\ref{prop:coverage} (Coverage is necessary)}\label{app:coverage}
The proof in Section~\ref{sec:discussion} constructs two environments that agree outside $A$ and differ on a future outcome inside $A$.  Because the data-collection policy never visits $A$, the induced trajectory laws are identical, establishing that distribution-free certification over unvisited states is impossible without additional structural or domain assumptions.

\subsection{Proof-assumption audit}\label{app:audit}
Table~\ref{tab:proof-audit} records every premise that does material work in the main claims, including the boundary cases that would otherwise silently invalidate an inference.  It is also the checklist used by the neural generative-PAC experiment.

\end{document}